\documentclass[letterpaper]{article}
\PassOptionsToPackage{hidelinks}{hyperref}  % ← add this line
\usepackage{flairs}
\usepackage{times}
\usepackage{soul}
\usepackage{url}
\usepackage[utf8]{inputenc}
\usepackage[small]{caption}
\usepackage{graphicx}
\usepackage{amsmath}
\usepackage{amsthm}
\usepackage{xcolor}
\theoremstyle{plain}
\newtheorem{theorem}{Theorem}
\newtheorem{lemma}{Lemma}

\theoremstyle{definition}
\newtheorem{definition}{Definition}
\newtheorem{assumption}{Assumption}

\usepackage{amsthm}
\usepackage{amssymb}
\usepackage{booktabs}
\usepackage{algorithm}
\usepackage{algorithmic}
\usepackage{subcaption}
\usepackage{tikz}
\usetikzlibrary{arrows.meta, positioning, fit, calc}

\begin{document}
% This file is an adoption of the style file for AAAI Press 
% proceedings, working notes, and technical reports.  This file is made 
% with minimal changes by explicit permission from AAAI.
\title{A Scalable Approach to Solving Simulation-Based Network Security Games}
\author{Michael Lanier and Yevgeniy Vorobeychik}

\maketitle
\begin{abstract}
We introduce \textbf{MetaDOAR}, a lightweight meta-controller that augments the Double Oracle / PSRO paradigm with a learned, partition-aware filtering layer and Q-value caching to enable scalable multi-agent reinforcement learning on very large cyber-network environments. MetaDOAR learns a compact state projection from per-node structural embeddings to rapidly score and select a small subset of devices (a top-$k$ partition) on which a conventional low-level actor performs focused beam search utilizing a critic agent. Selected candidate actions are evaluated with batched critic forwards and stored in an LRU cache keyed by a quantized state projection and local action identifiers, dramatically reducing redundant critic computation while preserving decision quality via conservative $k$-hop cache invalidation. Empirically, MetaDOAR attains higher player payoffs than SOTA baselines on large network topologies, without significant scaling issues in terms of memory usage or training time. This contribution provide a practical, theoretically motivated path to efficient hierarchical policy learning for large-scale networked decision problems. 
\end{abstract}
\section{Introduction}

Cybersecurity remains a fundamentally adversarial domain in which defenders must continually adapt to intelligent and resourceful attackers. Modeling this interaction requires balancing realism, scalability, and analytical rigor. While reinforcement learning (RL) and multi-agent learning methods have shown promise for sequential cyber-defense planning, their application to realistic network environments remains limited by the scale and structure of the underlying state and action spaces. Networks with thousands of interconnected devices, heterogeneous services, and partial observability pose severe computational challenges for learning algorithms that assume centralized control or dense equilibrium computation.

Game-theoretic approaches such as Double Oracle (DO) and its reinforcement learning extension, Policy-Space Response Oracles (PSRO), provide a principled mechanism for approximating equilibria in large multi-agent games. However, these methods rely on repeated simulation of the environment and evaluation of value functions for every strategy pair, causing their cost to grow super-linearly with network size. In cyber defense settings, each evaluation can involve complex topological dependencies and asynchronous device dynamics, making naive DO or PSRO intractable for large-scale environments.

Recent work has therefore shifted toward hierarchical or structured formulations that exploit the compositional nature of cyber networks. Recently, a software solution, CyGym,  was introduced as a scalable, simulation-based game-theoretic environment for cyber operations, coupling realistic network dynamics with partially observable stochastic-game semantics. CyGym enabled Double Oracle and reinforcement learning methods to be applied at moderate scale, but existing methods of solving in this simulator have remained bottlenecked by exhaustive exploration of device-level action combinations.

In this paper, we introduce MetaDOAR, a meta-controller designed to overcome these computational barriers. MetaDOAR augments the DO/PSRO paradigm with a learned, partition-aware filtering mechanism and cached critic reasoning. It learns compact state projections from per-node structural embeddings that allow it to identify a small subset of devices—the top-$k$ partition—where meaningful strategic decisions are concentrated. A low-level actor then performs focused beam search within this subset, guided by a centralized critic. To further reduce redundant computation, critic evaluations are stored in an LRU cache keyed by a quantized state projection and local action identifiers, with conservative $k$-hop invalidation to ensure correctness.

%We also extend the CyGym simulator to support efficient large-scale experimentation, including sparse network representations, snapshot-based resets, and standardized DO-compatible interfaces. These additions allow MetaDOAR to be evaluated rigorously under consistent conditions and make CyGym a more general platform for scalable multi-agent reinforcement learning research.

Empirical results demonstrate that MetaDOAR achieves higher player payoffs than state-of-the-art baselines on large CyGym topologies, while not exploding wall-clock training time or memory usage.\footnote{Code available at tobereleased.com} Together, these contributions advance the frontier of practical, theoretically grounded methods for learning adaptive, hierarchical policies in massive networked decision environments.

\section{Related Work}

\noindent\textbf{Game-Theoretic Security and Double Oracle Methods.}
Security games provide a standard framework for modeling strategic interaction between attackers and defenders, with applications to infrastructure protection, patrolling, and cyber defense~\cite{tambe2011security,korzhyk2011stackelberg}. 
Double oracle (DO) algorithms~\cite{mcmahan2003doublerobust,jain2011double} iteratively expand a restricted game via best responses and have been widely used to compute equilibria without enumerating the full strategy space.
Recent parametric and RL-based best-response methods adapt DO-style procedures to large or continuous strategy spaces in security and multi-agent domains~\cite{heinrich2015fsm}. ~\cite{lanier} instantiates a strong continuous-control double-oracle policy for realistic cyber ranges; ~\cite{sarkar2024geospatial} provides a heuristic-guided structure-aware search in the visual active search domain.
MetaDOAR is designed explicitly to \emph{build on} DOAR-style best responses while restructuring how candidate actions are selected and evaluated at scale. It keeps the original DOAR actor--critic and training loop intact, but places a lightweight selection and caching layer around them, enabling more tractable equilibrium-style reasoning on very large networks.

\smallskip
\noindent\textbf{RL for Cyber Defense and Networked Environments.}
Prior work has explored reinforcement learning for intrusion response, moving target defense, and autonomous penetration testing using Markov games or attack graphs~\cite{10.5555/2832249.2832322,schlenker2018deceiving,nguyen2021deep}. 
However, many of these methods are evaluated on relatively small networks or rely on action abstractions that avoid the full device-level combinatorics.
Recent cyber-range simulators and benchmarks expose much richer action spaces, but best-response and MARL baselines often assume full access to all $M$ devices at each step, leading to prohibitive latency and memory usage as $M$ grows.
MetaDOAR targets exactly this failure mode: it treats a strong DOAR actor--critic as a fixed inner solver and introduces a meta-controller that constrains the admissible device set without relaxing the underlying security model.

\smallskip
\noindent\textbf{Large Action Spaces and Efficient Structured Inference.}
Scaling RL to very large or combinatorial action spaces has motivated approaches based on action embeddings, factorized policies, ranking models, and attention over entities~\cite{dulacarnold2015largediscrete,tavakoli2018actionembeddings,wen2024reinforcinglanguageagentspolicy}.
Entity and graph-based RL methods exploit relational structure to focus computation on relevant objects or subgraphs~\cite{zambaldi2019relational}.
These techniques, however, typically co-design the policy and selection mechanism end-to-end, and are not constrained to preserve exact best-response behavior with respect to an existing critic.
MetaDOAR instead (i) builds compact per-device structural embeddings from IDs, degrees, and role-specific flags; (ii) uses a simple compatibility function with a global state embedding to select a small top-$k$ subset of devices; and (iii) employs an LRU Q-value cache keyed by projected state and device to avoid recomputing critic evaluations for unchanged regions.
The result is a best-response architecture tailored for large-scale cyber defense, rather than a generic large-action RL method.
Work on large–action reinforcement learning, for example \cite{dulacarnold2015largediscrete}, represents discrete actions with continuous embeddings and then uses a distance metric in that space to pick a small candidate set via $k$-nearest neighbours. MetaDOAR takes a different route: it learns a state-dependent scoring function over entity-indexed actions (devices or device subsets) and simply keeps the top-$K$ by score. In other words, MetaDOAR does not rely on any metric structure over the action space; it only needs a learnable ranking of actions given the current state, which is a more natural fit in structured domains such as cyber defence.

\smallskip
\noindent\textbf{MARL and Hierarchical Expert Policies for Cyber Defense.}
Multi-agent RL baselines such as IPPO and MAPPO~\cite{dewitt2020independentlearningneedstarcraft,yu2022surprising} represent standard approaches for large populations of agents, and hierarchical or expert-driven controllers introduce domain structure via handcrafted decompositions or expert sub-policies. Multi agent reinforcement learning (MARL) methods such as independent PPO,
centralized critic PPO, and hierarchical MARL architectures provide generic
mechanisms for learning in multi entity, multi role domains. Applied directly
to cyber ranges at enterprise scale, they face several intertwined obstacles.

First, observations are high dimensional. A useful state representation must
summarize vulnerabilities, services, credentials, topological position, and
defense posture for up to tens of thousands of devices. Second, the action
space is combinatorial. Actions select devices together with exploits,
configuration changes, or mitigations, frequently in mixed discrete and
continuous parameterizations. Third, simultaneous learning by attacker and
defender induces non stationarity, which complicates stable training.
Finally, straightforward MARL implementations tend to rely on large critics,
dense action evaluation, and substantial replay or rollout storage, which leads
to prohibitive memory usage and latency as $M$ increases.

These considerations motivate architectures that exploit network structure,
respect partial observability, and concentrate computation on strategically
relevant subsets of devices and actions, while preserving the game-theoretic
interpretation of best responses and equilibria.

While expressive, these methods often require scoring actions over all devices, leading to severe efficiency and memory issues in large networks.
Our empirical results highlight this gap: MARL and hierarchical baselines either hit out-of-memory conditions or exhibit orders-of-magnitude slower forward passes at $M=20{,}000$, while MetaDOAR remains compatible with DOAR's continuous-control best-responses and significantly improves latency and memory usage.

Hierarchical RL introduces temporal and structural abstraction through options and skills~\cite{sutton1999options,barto2003hrl}, feudal architectures~\cite{vezhnevets2017feudal}, and related multi-level controllers typically learn both high-level and low-level policies jointly and allow the high-level policy to specify goals, sub-policies, or task embeddings.
Subsequent work in hierarchical MARL extends these ideas to cooperative and competitive settings, often with centrally trained, decentrally executed structures.
In contrast, MetaDOAR is a \emph{meta-best-response} controller: it does not replace or retrain DOAR's inner policy, but instead learns  which devices the fixed best-response engine is allowed to act on.
This preserves the original continuous-control best response while introducing a structural prior on where computation should be spent.

\section{Preliminaries}
\label{sec:preliminaries}

\subsection{Cyber-Defense as a Partially Observable Stochastic Game}

We model the interaction between an attacker and a defender on an enterprise
network as a two-player zero-sum partially observable stochastic game (POSG).

Let $\mathcal{G} = (\mathcal{V}, \mathcal{E})$ denote the network, where each
node $i \in \mathcal{V}$ is a device such as a server, workstation, or critical
infrastructure component, and edges represent communication, authentication, or
trust relations. Let $M = |\mathcal{V}|$ be the number of devices.

A POSG instance is given by
\[
\left(
\mathcal{S},
\mathcal{A}^{\text{att}},
\mathcal{A}^{\text{def}},
P,
r,
\mathcal{O}^{\text{att}},
\mathcal{O}^{\text{def}},
O^{\text{att}},
O^{\text{def}},
\gamma
\right),
\]
where $\mathcal{S}$ is the state space, $\mathcal{A}^{\text{att}}$ and
$\mathcal{A}^{\text{def}}$ are the joint action spaces, $P$ is the transition
kernel, $r$ is the attacker reward function, $\mathcal{O}^{p}$ are observation
spaces, $O^{p}$ are observation kernels, and $\gamma \in (0,1]$ is the discount
factor.

At each timestep $t$, the true state $s_t \in \mathcal{S}$ encodes host
configurations, vulnerabilities, exploitability indicators, compromise status,
privileges, ownership, deployed defenses, and any relevant visibility or sensor
constraints. The attacker chooses
$a_t^{\text{att}} \in \mathcal{A}^{\text{att}}(s_t)$ and the defender chooses
$a_t^{\text{def}} \in \mathcal{A}^{\text{def}}(s_t)$, possibly based on their
own observation histories. The next state is drawn according to
\[
    s_{t+1} \sim P(\cdot \mid s_t, a_t^{\text{att}}, a_t^{\text{def}}),
\]
the attacker receives reward $r_t = r(s_t, a_t^{\text{att}}, a_t^{\text{def}})$,
and the defender receives $-r_t$, so the game is zero-sum.

Actions are structured and device indexed. A typical atomic action specifies a
target node and an operation such as
\[
(\texttt{node}=i,\; \texttt{type},\; \texttt{exploit/app/config}),
\]
and joint actions may involve selecting among many such candidates.
As $M$ and the catalog of exploits and configurations grow, the induced action
spaces become extremely large and evaluating rich action proposals for many
devices simultaneously becomes a primary computational and memory bottleneck.

\subsection{Double Oracle and DOAR}
\label{subsec:doar}
We use standard parametric policies $\pi_{\theta}(a | o)$ for attacker and defender. Given a fixed opponent mixture, a best response is any policy that maximizes expected discounted utility; we approximate this with deep actor–critic RL (more details and full notation in the appendix).
We consider a two-player zero-sum game between an attacker and a defender.
Let $\Pi_A$ and $\Pi_D$ denote the (typically huge) strategy sets of the
attacker and defender, and let $u(\pi_D,\pi_A)$ be the defender's expected
utility when the defender plays $\pi_D$ and the attacker plays $\pi_A$.

The classical double-oracle algorithm maintains finite strategy subsets
$\Pi_A^t \subseteq \Pi_A$ and $\Pi_D^t \subseteq \Pi_D$.
Each iteration $t$ consists of three steps:
\begin{enumerate}
  \item \textbf{Restricted equilibrium.} Compute a Nash equilibrium
        $(\sigma_D^t,\sigma_A^t)$ of the restricted game induced by
        $\Pi_D^t \times \Pi_A^t$.
  \item \textbf{Best responses.} Compute exact best-response strategies
        $\pi_D^{\text{BR}} \in \arg\max_{\pi_D \in \Pi_D}
        \mathbb{E}_{\pi_A \sim \sigma_A^t} [u(\pi_D,\pi_A)]$
        and
        $\pi_A^{\text{BR}} \in \arg\min_{\pi_A \in \Pi_A}
        \mathbb{E}_{\pi_D \sim \sigma_D^t} [u(\pi_D,\pi_A)]$.
  \item \textbf{Oracle expansion.} If either best response is outside the
        current restricted sets, add it:
        $\Pi_D^{t+1} \leftarrow \Pi_D^t \cup \{\pi_D^{\text{BR}}\}$,
        $\Pi_A^{t+1} \leftarrow \Pi_A^t \cup \{\pi_A^{\text{BR}}\}$,
        and repeat.
\end{enumerate}
When best responses are computed exactly, DO converges to a Nash equilibrium
of the full game~\cite{mcmahan2003doublerobust}.

In large stochastic games such as our cyber defense environment, computing
an exact best response is itself an intractable Markov decision problem.
DOAR replaces exact best responses with \emph{approximate} best responses
learned by deep reinforcement learning.

Fix an attacker mixture $\sigma_A^t$.
From the defender's perspective, this induces a discounted MDP
$(\mathcal{S},\mathcal{A},P^{\sigma_A^t},r,\gamma)$ in which the transition
kernel averages over attacker actions according to $\sigma_A^t$.
DOAR trains a parametric defender policy $\pi_{\theta}$ and critic
$Q_{\phi}$ (in our case, a DDPG-style actor--critic) to maximize expected
discounted defender utility against $\sigma_A^t$:
\[
  J(\theta)
  \;=\;
  \mathbb{E}\biggl[\sum_{t=0}^{\infty} \gamma^t r(s_t,a_t)\biggr],
  \quad
  a_t = \pi_{\theta}(s_t).
\]
The critic parameters $\phi$ are updated by minimizing a standard TD loss
\[
  \mathcal{L}_{\text{critic}}(\phi)
  =
  \mathbb{E}_{(s,a,r,s') \sim \mathcal{D}}
  \Bigl[
    \bigl(
      Q_{\phi}(s,a) - y(r,s')
    \bigr)^2
  \Bigr],
\]
where $\mathcal{D}$ is a replay buffer and
$y(r,s') = r + \gamma Q_{\phi^-}(s',\pi_{\theta^-}(s'))$ uses
slow-moving target networks $(\phi^-,\theta^-)$.
The actor parameters $\theta$ are updated by the deterministic policy
gradient
\[
  \nabla_{\theta} J(\theta)
  \approx
  \mathbb{E}_{s \sim \mathcal{D}}
  \bigl[
    \nabla_{\theta} \pi_{\theta}(s)
    \nabla_a Q_{\phi}(s,a)\big|_{a=\pi_{\theta}(s)}
  \bigr].
\]

After training converges (or after a fixed number of gradient steps),
we treat $\pi_{\theta}$ as an \emph{approximate} best response to
$\sigma_A^t$ and add it to the defender's strategy set.
The attacker side is handled symmetrically.
MetaDOAR builds directly on this DOAR framework: it accelerates
the approximate best-response training by constraining the defender's
action space via a learned meta-controller while keeping the surrounding
DOAR/PSRO loop unchanged.

\section{Approach}
\label{sec:approach}

We propose \textbf{MetaDOAR}, a hierarchical best-response architecture that augments a strong continuous-control double-oracle policy (DOAR) with a lightweight meta-controller. The meta-controller (i) learns which devices are most relevant, (ii) restricts DOAR's low-level actor--critic to a small top-$k$ subset of devices instead of all $M$ devices, and (iii) caches critic evaluations so that unchanged parts of the network do not trigger recomputation. This substantially improves latency and memory usage at scale.

We consider the standard double-oracle (DO) procedure with parametric best responses.
In our setting, DOAR maintains actor--critic policies for attacker and defender;
each new best-response (BR) policy is trained against the current equilibrium and
added to the strategy set.
In the original implementation, BR computation implicitly considers actions over
\emph{all} devices, so both policy evaluation and critic-based action decoding scale
poorly with the network size $M$.

MetaDOAR leaves DOAR's training and decoding code unchanged, but interposes
a meta-controller that decides which devices DOAR is allowed to act on at each step.

\subsection{Meta-Hierarchical Best Response}

For each role (attacker or defender), we instantiate a
\emph{MetaHierarchicalBestResponse} module.
Given a role-specific global observation $o$ and the current network, the module:

\begin{enumerate}
    \item Computes a relevance score for each device and selects a top-$k$ subset of indices.
          Unless specified otherwise, we set
          \[
              k = \max\{1, \alpha\lceil \log_{10}(\max(10, M)) \rceil\},
          \]
          so that $k$ grows only logarithmically in $M$, proportional to a hyperparameter $\alpha>0$. 
    \item Intersects this subset with a role-specific visibility mask, yielding
          a set of \emph{allowed} devices.
    \item Restricts DOAR's low-level actor--critic to propose and evaluate actions
          only on these allowed devices.
\end{enumerate}

To enable scalable, topology-aware selection, MetaDOAR constructs a compact
structural feature vector for each device $i \in \{0, \dots, M-1\}$:
\begin{itemize}
    \item a fixed random ID embedding $e_i \in \mathbb{R}^{d_{\text{id}}}$;
    \item the normalized graph degree of node $i$, obtained from the
          environment's subnet/graph representation;
    \item binary flags indicating whether the device is visible and/or
          attacker-owned (according to the simulator's state).
\end{itemize}
These quantities are concatenated into a structural feature
$x_i = [e_i;\,\deg(i);\text{visible}(i);\text{owned}(i)]$ and passed through
a small two-layer MLP (the \emph{node projector}) to obtain a structural
embedding $z_i \in \mathbb{R}^d$.
Fixing the node ID embeddings and only learning the state-side projector keeps the number of trainable parameters independent of $M$, so the meta-controller's training cost depends on $M$ only through $k$.
All node embeddings are stored in a cache
$E_{\text{cache}} \in \mathbb{R}^{M \times d}$.
Nodes are marked ``dirty'' when affected by actions or local changes (e.g.,
ownership flips); only dirty nodes are re-embedded, which amortizes the cost
to nearly constant per step.

The role-specific global observation $o$ is mapped to a low-dimensional
\emph{state embedding} $h(o) \in \mathbb{R}^d$ by a separate two-layer MLP
(the \emph{state projector}).
Device relevance scores are computed via a simple compatibility function:
\[
    \mathrm{score}(i \mid o)
    = z_i^\top h(o) + b,
\]
where $b$ is a learned scalar bias.
Given these scores and the visibility mask, we select the top-$k$ visible devices
via partial sort. It is a ranking model that predicts which devices are most promising for the downstream DOAR policy and is intentionally lightweight.

\smallskip
\noindent\textbf{Training the meta-controller.}
During best-response training, the meta-controller runs alongside DOAR.
At each time step $t$ it logs the global observation $o_t$, the top-$k$ subset
$\mathcal{K}_t \subseteq \{0,\dots,M-1\}$ selected by the scoring rule above, and the resulting scalar reward $r_t$.
We define a soft selection mask
\[
\mathrm{mask}_t(i) =
\begin{cases}
\frac{1}{|\mathcal{K}_t|}, & i \in \mathcal{K}_t,\\[2pt]
0, & \text{otherwise},
\end{cases}
\]
and form a scalar prediction
\[
    \hat{r}(o_t)
    = \sum_{i=0}^{M-1} \mathrm{mask}_t(i)\,\mathrm{score}(i \mid o_t).
\]
The parameters of the state projector, node projector, and bias $b$ are trained to minimize a mean-squared error loss
\[
    \mathcal{L}
    = \mathbb{E}\bigl[(\hat{r}(o_t) - r_t)^2\bigr],
\]
using replayed tuples $(o_t, \mathcal{K}_t, r_t)$ collected under the current DOAR
best-response policy.
The structural inputs $x_i$ (including the fixed ID embeddings $e_i$ and graph
features) act as topology-aware inputs to this learned scoring function.
This regression-style objective encourages the meta-controller to assign higher
scores to device subsets that correlate with higher returns under the current
DOAR behavior.

DOAR decodes continuous actor outputs into structured actions using its critic
for Q-value evaluation.
Naively enumerating actions over all devices and exploits is prohibitively
expensive for large $M$.
MetaDOAR replaces this with a restricted, cached evaluation over the selected
devices, maintaining an LRU cache of Q-values keyed by a quantized state
embedding and action tuple to further reduce evaluation cost.

\smallskip
\noindent\textbf{Approximation guarantees for MetaDOAR.}
We now consider the approximation error that comes purely from
MetaDOAR's pruning step, where the meta-controller limits the defender to a
learned top-$k$ subset of actions in each state.
To keep the argument clean, we pretend that the exact optimal best-response $Q$-function
exists and ignore all sources of statistical error: estimation noise,
function approximation, and any bias due to stale cache entries (these are
handled separately in the discussion that follows). As is standard in this space, we additionally assume that $s$ is decodable from $o$ asymptotically such that we can treat the POMDP as an MDP~\cite{nguyen2021deep}.

We view the defender's best-response problem against a fixed attacker mixture
as a discounted MDP $(\mathcal{S}, \mathcal{A}, P, r, \gamma)$.
The defender player follows without loss of generality.
Note that the underlying game is formally a partially observable stochastic game (POSG), so the true best-response problem is a POMDP.
We additionally assume the existence of a state estimator $\hat{s}_t = \phi(o_{0:t})$ such that the induced MDP in $\hat{s}_t$ has the same optimal best-response $Q$-function as the original POMDP, so the MDP bound applies in this encoded state space.

\begin{lemma}
\label{lem:qgap_to_valuegap_2}
Suppose a stationary policy $\pi$ satisfies
\[
\sup_{s \in \mathcal{S}}
\Bigl(
\max_{a \in \mathcal{A}(s)} Q^\star(s,a)
\;-\;
Q^\star\bigl(s, \pi(s)\bigr)
\Bigr)
\;\le\; \varepsilon_Q.
\]
Then its value function $V^\pi$ obeys
\[
\| V^\star - V^\pi \|_\infty
\;\le\;
\frac{\varepsilon_Q}{1 - \gamma}.
\]
\end{lemma}

The proof is included in the appendix.
Intuitively, at each state the policy $\pi$ loses at most $\varepsilon_Q$ in one-step $Q^\star$ value compared to the optimal action.
Because future rewards are geometrically discounted by $\gamma$, these per-step losses add up to at most a geometric series $\varepsilon_Q + \gamma\varepsilon_Q + \gamma^2\varepsilon_Q + \dots = \varepsilon_Q / (1-\gamma)$, which yields the claimed bound.

We now plug the MetaDOAR pruning rule into this lemma.

\begin{theorem}[MetaDOAR yields an $\varepsilon$–best response]
Fix the opponent mixture so a player has action set $\mathcal{A}(s)$ and discount $\gamma\in(0,1)$.
For each state $s$, the meta-controller selects a nonempty pruned set $S(s)\subseteq \mathcal{A}(s)$, and assume
$\pi_{\text{Meta}}(s)\in\arg\max_{a\in S(s)}Q^\star(s,a)$.
Define
\[
\Delta(s):=\max_{a\in\mathcal{A}(s)}Q^\star(s,a)-\max_{a\in S(s)}Q^\star(s,a)\] and \[
\Delta_{\max}:=\sup_{s\in\mathcal{S}}\Delta(s).
\]
Then
\[
\bigl\|V^\star - V^{\pi_{\text{Meta}}}\bigr\|_\infty \;\le\; \frac{\Delta_{\max}}{1-\gamma}.
\]
\end{theorem}

\begin{table*}[t]
\centering

\label{tab:meta-doar-all}
\scriptsize
\setlength{\tabcolsep}{3pt} % default is 6pt; tighten columns a bit
\begin{tabular}{rrrrrrrr}
\toprule
\textbf{Device Count} & \textbf{DOAR} & \textbf{HAGS} & \textbf{IPPO} & \textbf{MAPPO} & \textbf{HMARLExpert} & \textbf{HMARLMeta} & \textbf{MetaDOAR} \\
\midrule
10    & $24.88\pm24.86$ & $0.03\pm0.02$   & $36.22\pm20.83$ & $1.17\pm1.16$   & $-0.03\pm0.05$   & $50.01\pm25.00$  & $\mathbf{51.52\pm23.49}$ \\
50    & $0.008\pm0.002$ & $2.00\pm1.00$   & $2.00\pm1.00$   & $2.27\pm0.63$   & $1.00\pm1.00$    & $1.23\pm0.91$    & $\mathbf{2.97\pm0.04}$   \\
100   & $0.996\pm0.245$ & $0.50\pm0.50$   & $0.75\pm0.00$   & $0.96\pm0.21$   & $0.27\pm0.26$    & $0.26\pm0.25$    & $\mathbf{1.50\pm0.00}$   \\
1000  & $0.070\pm0.020$ & $0.070\pm0.020$ & $0.060\pm0.008$ & $0.058\pm0.005$ & $0.021\pm0.021$  & $0.020\pm0.012$  & $\mathbf{0.140\pm0.007}$ \\
10000 & $0.006\pm0.001$ & $0.002\pm0.002$ & $0.003\pm0.003$ & $0.003\pm0.003$ & $0.005\pm0.001$  & $0.007\pm0.001$  & $\mathbf{0.010\pm0.002}$ \\
\bottomrule
\end{tabular}

\caption{Average player utility per device in equilibrium. Larger is better. Experiment details are included in the appendix. }
\end{table*}

Proof is included in the appendix. This bound is a standard consequence of the contraction properties of Bellman operators in discounted MDPs \cite{puterman2014mdp},\cite{szepesvari2010rl}. Combined with the observation that strictly dominated strategies receive zero mass in the computed equilibrium,
the $\varepsilon$–best-response guarantee above provides a principled grounding for MetaDOAR’s pruning step, despite the meta-controller being implemented via a learned heuristic.

\smallskip
%\noindent\textbf{Hierarchical Selection and Caching.}
%In many realistic attack and defense scenarios, only a small subset of devices is
%strategically salient at any given time, for example machines near a compromise,
%high value assets, or topological chokepoints.
%Hierarchical designs formalize this by separating coarse selection from fine
%grained decision making.

%A meta controller first maps the current state to a ranking over devices and
%selects a small candidate subset, such as the top $k$ nodes.
%A base policy or best response module then enumerates and evaluates actions only
%for this subset.
%This reduces the cost of actor-critic evaluation while retaining the ability to
%represent rich policies, provided that the meta controller is lightweight,
%selection respects visibility and other constraints, and critic evaluations are
%reused through caching so that unchanged devices do not trigger redundant work.

%The architecture introduced here follows this general
%pattern to obtain scalable and semantically faithful best response computation in
%large enterprise networks.

%\smallskip
%\noindent\textbf{Chunked evaluation.}
%For each allowed device, we enumerate a limited set of candidate actions
%(e.g., capped exploit indices) and encode them using DOAR's existing
%\texttt{encode\_action} interface.
%Candidates are evaluated by the DOAR critic in small batches of configurable
%size to control memory usage.
%This yields Q-values $Q(s,a)$ for candidates localized to the selected devices.

\smallskip
\noindent\textbf{LRU Q-value cache.}
We maintain an LRU (last recently used) cache keyed by
\[
(\texttt{state\_key}, \texttt{node}, \texttt{action-type}, \texttt{exploit}, \texttt{app}),
\]
where \texttt{state\_key} is a hash of the quantized state embedding $h(s)$.
If a similar projected state recurs, we reuse stored Q-values instead of
invoking the critic.
Cache entries involving a node are invalidated when that node or any of its
$k$-hop neighbors is affected (using BFS on the current adjacency),
so devices remain logically ``alive'' without requiring a full
state recomputation.

For each node, we retain the candidate with the highest cached-or-evaluated
Q-value and return the resulting set of actions as the hierarchical best response.
If no candidate survives filtering, we fall back to a trivial safe action or to
the original DOAR proposal.

\smallskip
\noindent\textbf{Q-cache and approximation.}
The Q-value cache is an approximation device rather than a guaranteed error-preserving trick.
Whenever we reuse a cached value, we are implicitly assuming that the projected state and the local neighborhood around the acted-on device have not changed enough to alter the critic’s judgment in a meaningful way.
This will not always hold exactly, however in preferentially connected networks (such as the ones in the experiments and common to networked systems) it often does. We do not claim that MetaDOAR recovers the exact DOAR best response under caching.

To keep this approximation from drifting too far, we combine several simple safeguards:
each cache entry is tagged with a short time-to-live and is discarded after a bounded number of uses; we periodically flush the entire cache; and we invalidate entries whose target device lies within a $k$-hop neighborhood of any node that has just been modified.
In addition, a small fraction of lookups are intentionally recomputed from scratch, which serves as a lightweight spot check on stale entries.

These mechanisms are heuristic, and we do not provide a formal error bound on the induced bias.
Our goal is more modest: reduce redundant critic calls while keeping the induced change in behaviour small enough that MetaDOAR remains competitive with, or better than, the underlying DOAR baseline.
In ablation experiments where we tighten these reuse rules, we observe only minor differences in defender payoffs, while wall-clock time and memory usage increases substantially.

\smallskip
\noindent\textbf{Training the Meta-Controller.}
The meta-controller is trained online using a small replay buffer.
After each environment step, we record
$(o, K_{\text{sel}}, r, o', \mathrm{done})$,
where $K_{\text{sel}}$ is the subset of devices selected by MetaDOAR
and $r$ is the realized reward.

Periodically, we:
\begin{enumerate}
    \item sample a minibatch from the buffer;
    \item construct, for each transition, a soft mask over devices that places
          uniform mass on $K_{\text{sel}}$;
    \item compute a scalar prediction
          \[
              \hat{r}(o) =
              \sum_i \mathrm{mask}_i \cdot
              \bigl(z_i^\top h(o) + b\bigr),
          \]
          and minimize the mean squared error between $\hat{r}(o)$ and $r$.
\end{enumerate}

A convenient way to read this objective is as a crude value factorization:
the scores $q_i(o)$ are encouraged to behave like per-device contributions
to the outcome under the current DOAR behaviour.
Whenever a transition with a high return is seen, the devices that happened
to be in $K_{\text{sel}}$ are nudged toward higher scores; when the return
is poor, the same devices are nudged down.
The soft mask is just a differentiable way of saying “only the currently
selected devices matter for this update”.
Over time, devices that repeatedly co-occur with good outcomes tend to receive
higher scores and are more likely to appear in future top-$k$ sets.

This training rule is deliberately lightweight and should be viewed as
heuristic rather than as a principled solution to credit assignment in
partially observable games.
We do not claim convergence guarantees for the meta-controller, especially
in the presence of non-stationary opponents.
To keep its behaviour stable in practice, we restrict the meta-network to
a small MLP, use conservative learning rates, and maintain a slowly updated
target network for the state embedding.

%This regression-style objective encourages the meta-controller to assign higher
%scores to device subsets that correlate with higher returns under the current
%DOAR behavior.
%We optimize with Adam and use a slow-moving target network for the state projector
%for additional stability.
%The meta-network is intentionally small so that its overhead remains negligible
%relative to the savings it unlocks.

%During the forward pass, the meta-controller computes per-node scores
%$\mathrm{score}(i \mid s) = z_i^\top h(s) + b$ and selects the top-$k$
%visible devices. The DOAR actor--critic is then restricted to proposing
%and evaluating actions only on this subset, while reusing cached critic
%evaluations for unchanged nodes.
Consequently, our approach retains the underlying DOAR continuous-control best-response machinery, but applies it through a learned top-$k$ selector and a cached critic interface. This yields an \emph{approximate} DOAR-style best response that is much cheaper to evaluate in large network settings.
Since MetaDOAR restricts the admissible actions and reuses stored Q-values, it does not compute an exact best response; instead, it aims to approximate the DOAR solution while substantially reducing computation. Because DO converges when each iteration uses (approximate) best responses, MetaDOAR’s $\varepsilon$–best responses yield an $\varepsilon$–Nash equilibrium rather than an exact one.

\begin{table}[t]
\centering
\begin{tabular}{rccc}
\hline
$k$ &
Utility &
Wall Time (ms) &
RAM (MB) \\
\hline
1 &
$135.98 \pm 0.78$ &
$1.28 \pm 0.01$ &
$1405.81 \pm 0.90$ \\
4 &
$127.60 \pm 7.50$ &
$1.50 \pm 0.12$ &
$1417.25 \pm 4.43$ \\
10 &
$138.86 \pm 3.75$ &
$1.56 \pm 0.02$ &
$1416.75 \pm 3.81$ \\

\hline
\end{tabular}
\caption{Effect of the $k$-hop invalidation radius hyperparameter for the $Q$ cache on average player utility and payoff-matrix construction cost for a 1000 device network. Higher utility is better. Lower wall time and RAM is better.}
\label{tab:k_invalidation}
\end{table}

\section{Experiments}

We evaluate MetaDOAR in three sets of experiments, all conducted in the Volt Typhoon CyberDefenseEnv instantiation of CyGym. Unless otherwise noted, we follow the experimental protocol of \cite{lanier}: we reuse the same Double Oracle training loop sans the far apart reset check, reward structure, and optimization hyperparameters, and defer the full list of environment and training settings to the appendix. To study scalability, we instantiate networks with increasing device counts (e.g., 10, 50, 100, 1000, and 10000 devices) and scale key environment parameters approximately linearly with the number of devices (for example, the initial number of attacker-owned devices and other network-level rates and probabilities), so that larger networks represent proportionally larger but qualitatively similar threat environments.
We evaluate the performance of the MetaDOAR best response, comparing convergence speed and final payoffs to the original DOAR best response under matched training budgets. Finally, we demonstrate that MetaDOAR scales to increasingly large networks by comparing its end-to-end performance and computational cost to baseline best-response methods across the full range of network sizes.

We analyze the performance of best response algorithms with the overall Double Oracle game solver. We note that MetaDOAR significantly out performs existing SOTA methods for larger networks and perform as well as existing methods for smaller networks. We compare to DOAR itself, along side MARL baselines such as \cite{dewitt2020independentlearningneedstarcraft} and \cite{yu2022surprising}. We also compare to other hierarchical approaches \cite{sarkar2024activegeospatialsearchefficient}, \cite{singh2025hierarchicalmultiagentreinforcementlearning}. Additional abalations and discussion are included in the appendix. 

Across all settings, the experiments show that MetaDOAR maintains or improves DOAR’s strategic performance while making large games computationally tractable. In equilibrium, MetaDOAR achieves the highest average player utility per device at every network size from $10$ up to $10{,}000$ devices, often roughly doubling DOAR’s payoff at large scale and outperforming all MARL and hierarchical baselines (IPPO, MAPPO, HAGS, HMARLExpert, HMARLMeta). At the same time, hierarchical top-$k$ selection and the Q-value cache keep payoff-matrix construction wall time in the $\sim 1$--$2$ ms range and memory around $\sim 1.4$ GB, with only minor variation as $M$ increases, so MetaDOAR remains tractable where several baselines become unstable or prohibitively slow (see the appendix). Ablations on the $k$-hop invalidation radius show that small radii already recover most of the utility gains, with only modest increases in wall time and RAM as $k$ grows, indicating that the cache and selection mechanism are effective without aggressive invalidation.

%\paragraph{Outer-loop behavior.}
%Our goal in this work is to make each best-response computation scalable; the outer double-oracle loop (meta-strategy update and stopping condition) is identical to the DOAR setup of \cite{lanier}. For completeness, the appendix reports the defender value of the empirical meta-strategy across DO iterations together with the corresponding best-response payoffs. The trajectories for MetaDOAR and DOAR are qualitatively similar in terms of convergence behavior, but MetaDOAR achieves each iteration at substantially lower wall-clock cost.

%\paragraph{Limitations.}
%Our empirical study focuses on a single, but realistic, cyber-range configuration: the Volt Typhoon instantiation of CyGym with one family of topology generators and reward structures \cite{lanier}. We do not claim that the specific numerical improvements will transfer unchanged to all network models or visibility regimes. While this limits results, there are currently no other cyber ranges available that scale to thousands of devices. Applying MetaDOAR to qualitatively different graphs (e.g., grids vs.\ preferential-attachment networks), alternative observation models, and other cyber ranges is an important direction for future work.

\section{Conclusion}
\label{sec:conclusion}

This work introduced \textbf{Metadoar}, a hierarchical meta-controller.
Instead of redesigning the underlying RL algorithms or relaxing the best-response formulation, Metadoar is layered on top of a strong DOAR policy and (i) selects a small, topology-aware subset of relevant devices, (ii) constrains critic-guided decoding to that subset, and (iii) amortizes evaluation via a structured Q-value cache.

Metadoar yields an approximate best-response operator that remains faithful to the original double-oracle procedure while making it tractable on networks an order of magnitude larger than those targeted by prior work.

Empirically, Metadoar consistently reduces end-to-end forward-pass latency and peak memory usage compared to the underlying DOAR implementation, while maintaining compatibility with the existing training and execution stack.
On networks with up to 10{,}000 devices, Metadoar achieves fast, stable best-response computation where several multi-agent RL baselines either incur prohibitive latency or exhaust GPU memory.

\bibliographystyle{named}
\bibliography{flairs}
\clearpage
\onecolumn
\section{Appendix}

\section{Parametric policies, best responses, Nash equilbria}

In a POSG, a general strategy for player $p$ maps its information history to a
distribution over actions. In practice, we use parametric policies
$\pi_{\theta_p}(a \mid o)$, often Markovian in a learned representation of the
agent's observation or belief state.

Given an opponent strategy $\sigma_{-p}$, the value of a policy $\pi_p$ for
player $p$ is
\[
    V_p(\pi_p, \sigma_{-p})
    =
    \mathbb{E}\!\left[
        \sum_{t=0}^{T} \gamma^t r_t^p
        \,\middle|\,
        \pi_p, \sigma_{-p}
    \right],
\]
where $r_t^{\text{att}} = r_t$ and $r_t^{\text{def}} = -r_t$.
A (possibly approximate) best response for player $p$ against $\sigma_{-p}$ is
any policy
\begin{equation}
    \pi^{\mathrm{BR}}_p \in
    \operatorname*{arg\,max}_{\pi} V_p(\pi, \sigma_{-p}).
    \label{eq:br-def-posg}
\end{equation}

In a two player zero sum game, a Nash equilibrium is a pair of strategies
$(\sigma_{\text{att}}^*, \sigma_{\text{def}}^*)$ such that
\[
    V_{\text{att}}(\sigma_{\text{att}}^*, \sigma_{\text{def}}^*)
    \ge
    V_{\text{att}}(\pi_{\text{att}}, \sigma_{\text{def}}^*)
    \quad \forall \pi_{\text{att}},
\]
\[
    V_{\text{def}}(\sigma_{\text{att}}^*, \sigma_{\text{def}}^*)
    \ge
    V_{\text{def}}(\sigma_{\text{att}}^*, \pi_{\text{def}})
    \quad \forall \pi_{\text{def}}.
\]
Equivalently, neither player can improve its value by unilaterally deviating.
Best-response computation is thus the central primitive for equilibrium
approximation in this setting.

In deep RL based formulations, best responses are parametric and obtained by
training neural policies against fixed or slowly updated opponent mixtures.
The cost of computing each such best response is amplified in cyber-defense
domains by the size and structure of the underlying action space.

\paragraph{Experimental protocol and baseline configuration.}
All error bars are reported over 2 seeds. All methods share the same observation encoding, reward function, and simulator implementation (Volt Typhoon instantiation of CyGym) as well as network topology. For each device count $M$, we train DOAR, MetaDOAR, IPPO, MAPPO, HMARLExpert, and HMARLMeta for the same number of environment steps and double-oracle iterations, with identical episode length and rollout parallelism. Hyperparameters for DOAR and the MARL baselines follow the settings recommended in prior work or were tuned once at $M = 10$ and then reused across larger networks; we do not retune per $M$ to keep the comparison controlled. MetaDOAR is only applied to DOAR: none of the baselines use action pruning, top-$k$ selection, or Q-value caching, so they operate on the full device set at each step. All runs are executed on the same GPU/CPU hardware.

\label{sec:appendix}

\captionsetup[table]{skip=2pt}
\begin{table}[!htbp]
  \centering
  \caption{Environment and best-response hyperparameters used in our experiments.}
  \label{tab:env_ddpg_params}
  \begin{tabular*}{\textwidth}{@{\extracolsep{\fill}} llll}
    \toprule
    \multicolumn{4}{l}{\textbf{Environment Parameters}} \\
    \midrule
    Devices ($M$)          & $10, 100, 1000, 10000, 20000$  & Steps / episode           & $100$ \\
    Max\_network\_size     & $M + 10$                       & Initial compromised ratio & $0.4$ \\
    Zero day               & False                          & fast\_scan                & True \\
    work\_scale            & $1.0$                          & comp\_scale               & $30$ \\
    num\_attacker\_owned   & $\max\!\bigl(1,\lfloor 0.05 M + 0.5 \rfloor\bigr)$ & def\_scale & $1.0$ \\
    defaultversion         & $1.0$                          & default\_mode             & $1$ \\
    default\_high          & $3$                            & $\gamma$                  & $0.99$ \\
    DO min iterations      & $10$                           & DO max iterations         & $15$ \\
    \midrule
    \multicolumn{4}{l}{\textbf{Best-Response Hyperparameters}} \\
    reward\_scale          & $1.0$                          & max\_grad\_norm           & $0.5$ \\
    soft-$\tau$ (target update) & $0.01$                    & replay buffer capacity    & $100{,}000$ \\
    \midrule
    \multicolumn{4}{l}{\textbf{MetaDOAR Meta-Controller}} \\
    pruning coefficient $\alpha$ & $1$                      & $k$-hop invalidation radius & $1$ \\
    \midrule
    \multicolumn{4}{l}{\textbf{Defender Agent}} \\
    actor\_lr              & $0.001$                        & critic\_lr                & $0.01$ \\
    critic architecture    & [1229$\rightarrow$128$\rightarrow$128$\rightarrow$1] & & \\
    greedy-$K$             & $5$                            & greedy-$\tau$             & $0.5$ \\
    noise\_std             & $0.1$                          &                           &      \\
    $\lambda_{\text{events}}$ & $0.7$                      & $p_{\text{add}}$          & $0.1$ \\
    \midrule
    \multicolumn{4}{l}{\textbf{Attacker Agent}} \\
    actor\_lr              & $0.001$                        & critic\_lr                & $0.01$ \\
    critic architecture    & [1004$\rightarrow$128$\rightarrow$128$\rightarrow$1] & & \\
    greedy-$K$             & $5$                            & greedy-$\tau$             & $0.5$ \\
    noise\_std             & $0.1$                          &                           &      \\
    $\lambda_{\text{events}}$ & $0.7$                      & $p_{\text{add}}$          & $0.1$ \\
    \bottomrule
  \end{tabular*}
\end{table}
\twocolumn

\twocolumn
\begin{figure}[t]
  \centering
  \includegraphics[scale=0.55]{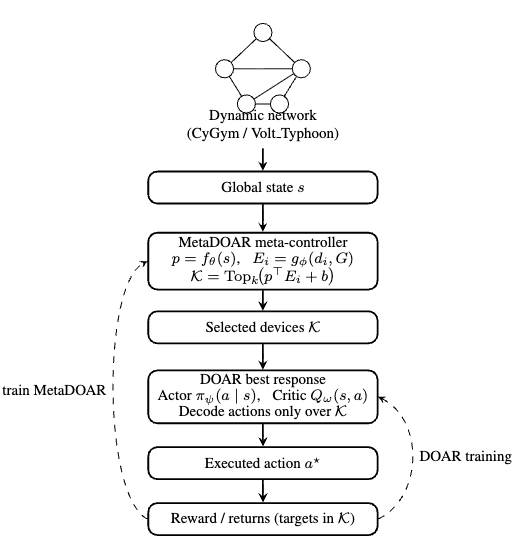}
  \caption{High-level view of MetaDOAR. The meta-controller scores devices
  and selects a small subset $\mathcal{K}$; DOAR then learns a best response
  while its action decoding is restricted to $\mathcal{K}$.}
  \label{fig:MetaDOAR_overview}
\end{figure}
\begin{table}[!htbp]
  \centering
  \caption{MetaDOAR cache and selection hyperparameters (used for all reported experiments).}
  \label{tab:metadoar_cache_params_2}
  \begin{tabular}{ll}
    \toprule
    \textbf{Parameter} & \textbf{Value} \\
    \midrule
    Top-$k$ scaling $\alpha$        & $1$ \\
    $k$-hop invalidation radius     & $1$ \\
    State projection quantization   & round to $3$ decimal places \\
    LRU cache capacity              & $50{,}000$ entries \\
    Cache time-to-live (TTL)       & $50$ cache steps \\
    Cache flush interval            & $200$ cache steps \\
    Random re-eval probability      & $0.01$ per cache hit \\
    \bottomrule
  \end{tabular}
\end{table}

Unless otherwise noted, we use the cache configuration in Table ~\ref{tab:metadoar_cache_params_2} use a 50k-entry LRU, state projections rounded to three decimals, a TTL of 50 cache steps, periodic flush every 200 steps, and a $1$-hop invalidation radius.

\section{Oracle Convergence}
\begin{figure}[t]
    \centering
    \includegraphics[width=0.48\textwidth]{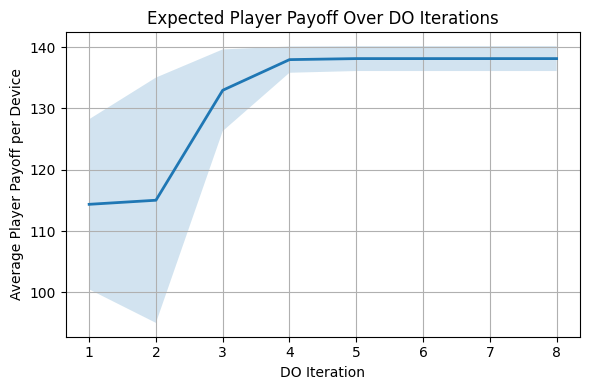}
    \caption{Expected player payoff over Double Oracle iterations for the 1000-device setting. 
    The solid line shows the mean equilibrium payoff per device, averaged over three random seeds. 
    The shaded region denotes a one-standard-error confidence band ($\pm 1$\,SE) across seeds.}
    \label{fig:do_convergence}
\end{figure}

Figure~\ref{fig:do_convergence} summarizes the outer-loop behavior of Double Oracle in the Volt Typhoon environment. The curve shows the expected equilibrium payoff per device, where we first average attacker and defender utilities at each iteration, normalize by the number of devices, and then average across three random seeds; the shaded region depicts a $\pm 1$ standard error band over those seeds. Across all runs, the equilibrium value improves rapidly in the first few iterations (1–3) and then stabilizes at roughly $0.1$ utility per device, with only small seed-to-seed variation thereafter. This pattern suggests that the DO outer loop converges in a handful of iterations on this game instance and that the learned equilibrium value is robust to initialization noise.
\section{Theoretical Considerations for MetaDOAR}

\begin{assumption}[Fixed-opponent MDP]
\label{ass:fixed_opponent}
The attacker commits to a stationary mixed strategy. Consequently, from the
defender's perspective, the environment is a Markov decision process with
(i) state space $\mathcal{S}$,
(ii) per-state action sets $\mathcal{A}(s)$ (the \emph{full} action set),
(iii) transition kernel $P(\cdot \mid s,a)$ and reward function $r(s,a)$, and
(iv) discount factor $\gamma \in (0,1)$.
\end{assumption}

\begin{assumption}[Bounded rewards]
\label{ass:bounded_rewards}
There is a finite constant $R_{\max}$ such that
$|r(s,a)| \le R_{\max}$ for every state--action pair $(s,a)$.
\end{assumption}

Let $V^\star$ and $Q^\star$ be the optimal value and $Q$-functions for this
MDP (i.e., optimal w.r.t.\ the full action sets $\mathcal{A}(s)$).

\begin{definition}[MetaDOAR pruning]
\label{def:pruning}
For each state $s \in \mathcal{S}$, the MetaDOAR meta-controller chooses a
non-empty subset $S(s) \subseteq \mathcal{A}(s)$ of \emph{allowed} actions.
The low-level actor (DOAR) then selects an action
$\pi_{\text{Meta}}(s) \in S(s)$. We assume that both the mapping
$s \mapsto S(s)$ and the policy $\pi_{\text{Meta}}$ are stationary.
\end{definition}

Given a pruning rule $S(\cdot)$ and the optimal $Q^\star$, we quantify the
loss of pruning at a state $s$ by
\[
\Delta(s)
\;:=\;
\max_{a \in \mathcal{A}(s)} Q^\star(s,a)
\;-\;
\max_{a \in S(s)} Q^\star(s,a).
\]
By definition $\Delta(s) \ge 0$, and $\Delta(s) = 0$ whenever at least one
optimal action survives in $S(s)$.

Since the rewards are finite, $\Delta_{\max}$ is maximized at the largest possible value drop at a single
decision point that can be caused by throwing away the globally best action.

We will use a standard connection between one-step suboptimality in $Q^\star$
and the resulting loss in value for discounted MDPs.

\begin{lemma}[From $Q^\star$-gaps to value loss]
\label{lem:qgap_to_valuegap}

Suppose a stationary policy $\pi$ satisfies
\[
\sup_{s \in \mathcal{S}}
\Bigl(
\max_{a \in \mathcal{A}(s)} Q^\star(s,a)
\;-\;
Q^\star\bigl(s, \pi(s)\bigr)
\Bigr)
\;\le\; \varepsilon_Q.
\]
Then its value function $V^\pi$ obeys
\[
\| V^\star - V^\pi \|_\infty
\;\le\;
\frac{\varepsilon_Q}{1 - \gamma}.
\]
\end{lemma}

\begin{proof}
Let $T$ be the optimal Bellman operator and $T^\pi$ the Bellman operator for
policy $\pi$:
\[
[T V](s) = \max_{a \in \mathcal{A}(s)}
\Bigl( r(s,a) + \gamma \mathbb{E}[V(s') \mid s,a] \Bigr)\] and
\quad
\[[T^\pi V](s) = r\bigl(s,\pi(s)\bigr) + \gamma \mathbb{E}[V(s') \mid s,\pi(s)].
\]
We have $V^\star = T V^\star$ and $V^\pi = T^\pi V^\pi$. Thus
\[
V^\star - V^\pi
= T V^\star - T^\pi V^\pi
= (T V^\star - T^\pi V^\star) + (T^\pi V^\star - T^\pi V^\pi).
\]
Taking the $\ell_\infty$ norm and using that $T^\pi$ is a
$\gamma$-contraction under $\|\cdot\|_\infty$,
\[
\|V^\star - V^\pi\|_\infty
\le
\|T V^\star - T^\pi V^\star\|_\infty
+ \gamma \|V^\star - V^\pi\|_\infty.
\]
After rearranging we obtain
\[
(1-\gamma)\|V^\star - V^\pi\|_\infty
\le
\|T V^\star - T^\pi V^\star\|_\infty.
\]
For each $s$,
\[
[T V^\star](s) - [T^\pi V^\star](s)
=
\max_{a \in \mathcal{A}(s)} Q^\star(s,a)
- Q^\star\bigl(s,\pi(s)\bigr),
\]
so the right-hand side equals
$\sup_s \bigl(\max_{a \in \mathcal{A}(s)} Q^\star(s,a) - Q^\star(s,\pi(s))\bigr)
\le \varepsilon_Q$ by assumption.
Dividing by $(1-\gamma)$ yields the claim.
\end{proof}

\begin{theorem}[MetaDOAR as an $\varepsilon$--best response to a fixed mixture]
\label{thm:metadoar_eps_br}
The MetaDOAR
policy $\pi_{\text{Meta}}$ satisfies
\[
\| V^\star - V^{\pi_{\text{Meta}}} \|_\infty
\;\le\;
\frac{\Delta_{\max}}{1 - \gamma}.
\]
Equivalently, against the fixed opponent mixture (Assumption
\ref{ass:fixed_opponent}), $\pi_{\text{Meta}}$ is an $\varepsilon$--best response
w.r.t.\ the \emph{full} action sets $\mathcal{A}(s)$ with
\[
\varepsilon \;=\; \frac{\Delta_{\max}}{1 - \gamma}.
\]
\end{theorem}

\begin{proof}
Fix any state $s$. By Definition \ref{def:pruning}, $\pi_{\text{Meta}}(s)\in S(s)$,
and therefore
\[
Q^\star\bigl(s,\pi_{\text{Meta}}(s)\bigr)
\;\ge\;
\max_{a \in S(s)} Q^\star(s,a).
\]
Subtracting both sides from $\max_{a \in \mathcal{A}(s)} Q^\star(s,a)$ gives
\[
\max_{a \in \mathcal{A}(s)} Q^\star(s,a)
- Q^\star\bigl(s,\pi_{\text{Meta}}(s)\bigr)\]
\[\le\]
\[\max_{a \in \mathcal{A}(s)} Q^\star(s,a)
- \max_{a \in S(s)} Q^\star(s,a)
\;=\;
\Delta(s)
\;\le\;
\Delta_{\max}.
\]
Taking $\sup_{s \in \mathcal{S}}$ of the left-hand side yields
\[
\sup_{s \in \mathcal{S}}
\Bigl(
\max_{a \in \mathcal{A}(s)} Q^\star(s,a)
- Q^\star\bigl(s,\pi_{\text{Meta}}(s)\bigr)
\Bigr)
\;\le\;
\Delta_{\max}.
\]
Applying Lemma \ref{lem:qgap_to_valuegap} with $\varepsilon_Q := \Delta_{\max}$
gives
\[
\| V^\star - V^{\pi_{\text{Meta}}} \|_\infty
\;\le\;
\frac{\Delta_{\max}}{1-\gamma},
\]
which proves the claim.
\end{proof}

\section{Scalability Experiments}

We examine various network scaling metrics. We note that MetaDOAR is typically on par with DOAR and HAGS, and significantly better than IPPO, MAPPO, and HMARLMeta. We note that MetaDOAR does not explode in complexity as network sizes grow and is typically as fast or faster than existing SOTA. We note that MARL methods completely fail with out of memory errors at 20000 device counts.

These results show us that MetaDOAR provides better performance with large scaling cost, in contrast to other methods.
\begin{figure*}[t]
    \centering

    % Row 1 (3 plots)
    \begin{subfigure}{0.32\textwidth}
        \centering
        \includegraphics[width=\linewidth]{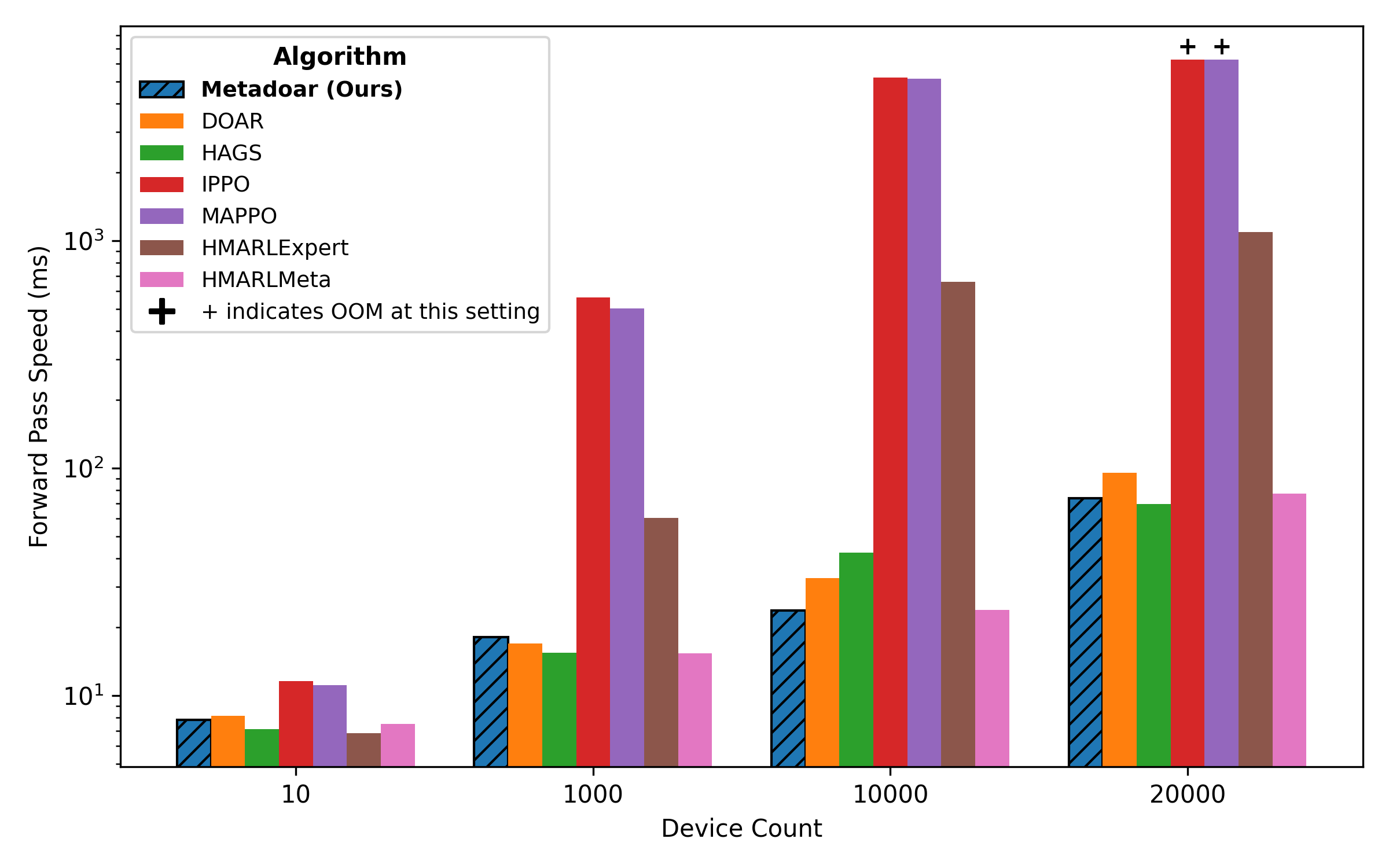}
        \caption{Forward Pass Wall Time}
        \label{fig:scaling-forward}
    \end{subfigure}
    \hfill
    \begin{subfigure}{0.32\textwidth}
        \centering
        \includegraphics[width=\linewidth]{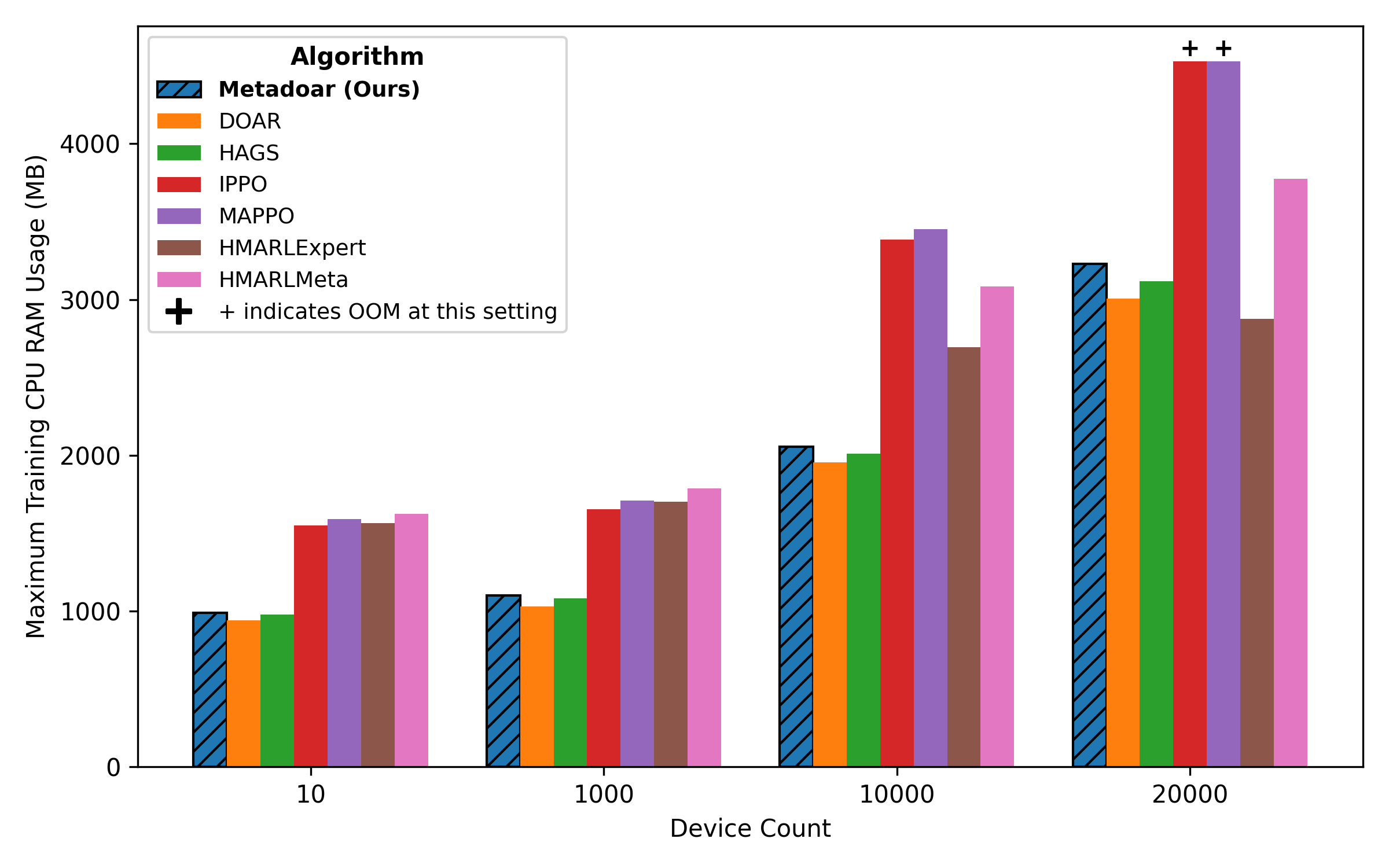}
        \caption{Max Training CPU RAM}
        \label{fig:scaling-train-ram}
    \end{subfigure}
    \hfill
    \begin{subfigure}{0.32\textwidth}
        \centering
        \includegraphics[width=\linewidth]{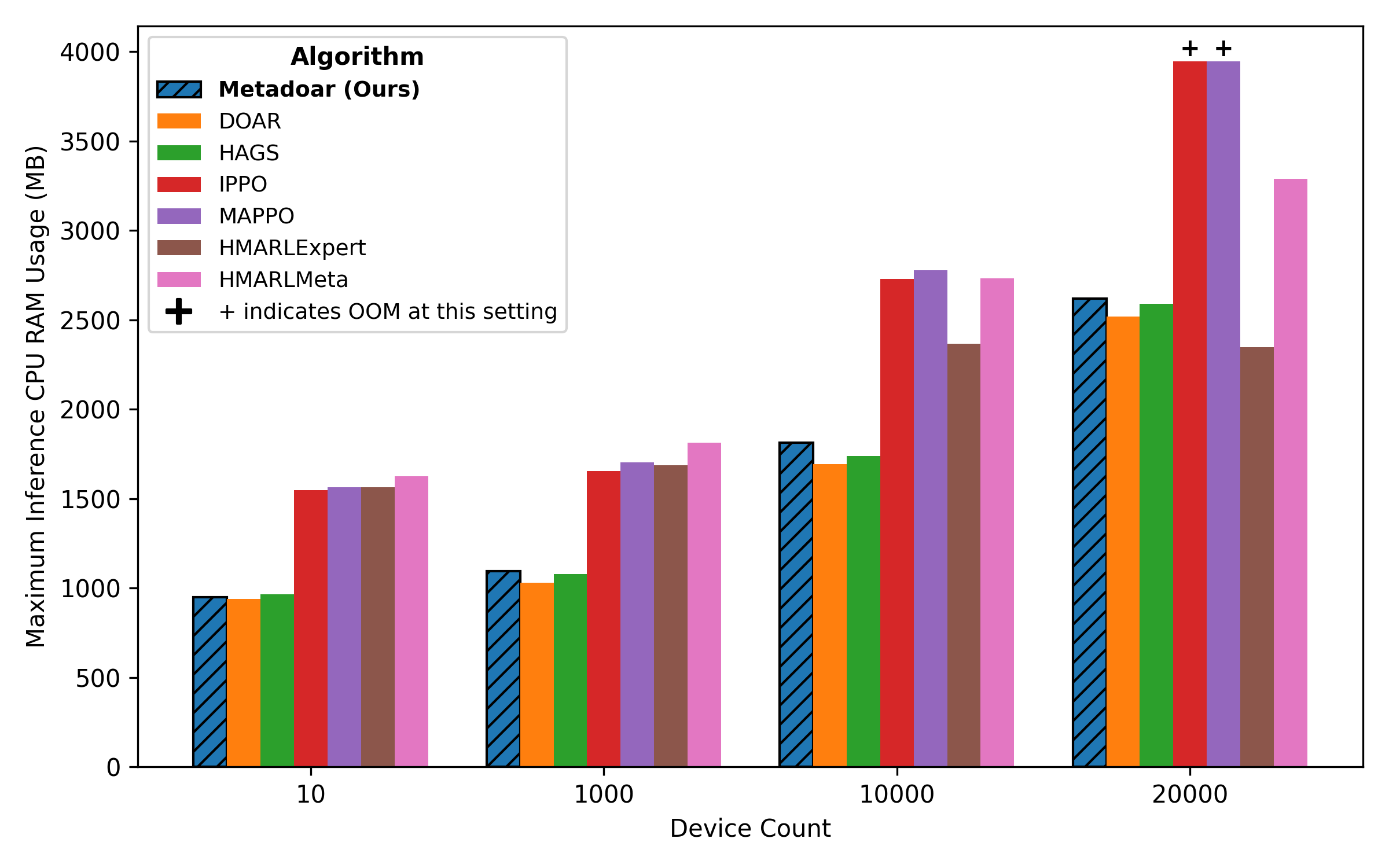}
        \caption{Max Inference CPU RAM}
        \label{fig:scaling-infer-ram}
    \end{subfigure}

    \vspace{0.6em}

    % Row 2 (2 plots, centered)
    \begin{subfigure}{0.38\textwidth}
        \centering
        \includegraphics[width=\linewidth]{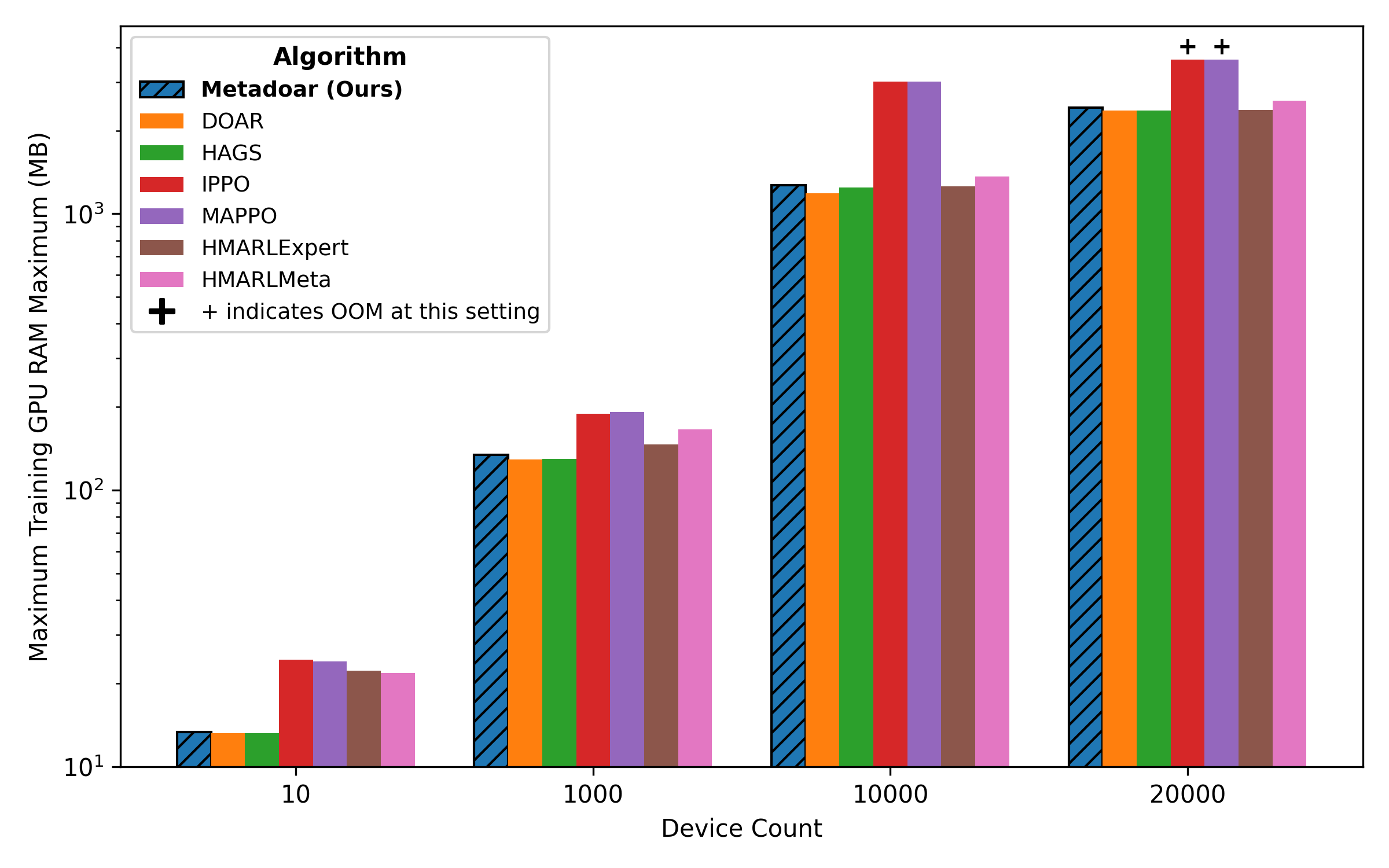}
        \caption{Max Training GPU RAM}
        \label{fig:scaling-train-cuda}
    \end{subfigure}
    \hspace{0.04\textwidth}
    \begin{subfigure}{0.38\textwidth}
        \centering
        \includegraphics[width=\linewidth]{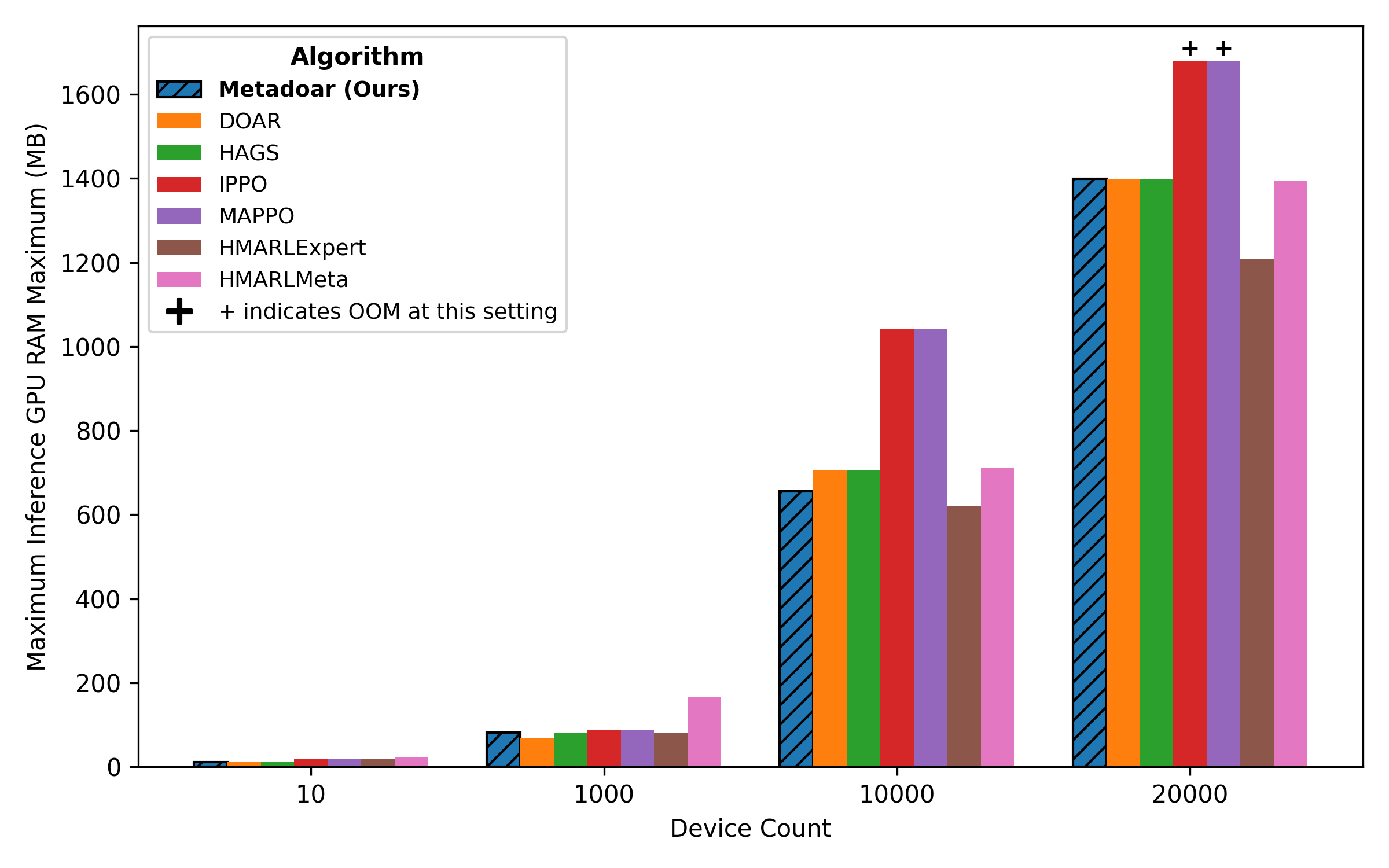}
        \caption{Max Inference GPU RAM}
        \label{fig:scaling-infer-cuda}
    \end{subfigure}

    \caption{
        Scalability of \textbf{MetaDOAR (Ours)} vs.\ baselines across 10, 100,  10000, and 20000 devices
        for forward-pass latency and peak memory usage.
        MetaDOAR (Ours) is highlighted in the legend of each plot. Smaller is better.
    }
    \label{fig:MetaDOAR-scaling}
\end{figure*}

\section{\texorpdfstring{$\alpha$}{alpha} hyperparameter ablation}

\paragraph{Effect of the pruning parameter $\alpha$.}
The meta-controller in MetaDOAR does not work with all devices at once. Instead, for a network with $M$ devices it picks
\[
k \;=\; \left\lceil \alpha \cdot \log_{10}\bigl(\max(10, M)\bigr) \right\rceil
\]
candidate devices per decision, capped by the number of visible devices. Small values of $\alpha$ make the controller highly selective, while larger values increase $k$ until almost every device is allowed through.

Table~\ref{tab:alpha_sweep} reports a sweep over $\alpha$. For the smaller settings, $\alpha = 1$ and $\alpha = 5$, defender utility is essentially unchanged. Pushing to a very large value, $\alpha = 50$, lowers the average payoff, suggesting that once the meta-controller is allowed to expose nearly all devices, we begin to lose the benefit of focusing computation on a small, high-impact subset. In that regime the behavior approaches vanilla DOAR in terms of which actions are available, but still pays the overhead of the additional machinery.

These results are consistent with the intended role of $\alpha$: moderate pruning (here, $\alpha \in \{1,5\}$) preserves performance while enforcing a strong structural prior, whereas very large $\alpha$ erodes that prior and yields slightly worse payoffs. In the main experiments we fix $\alpha = 1$, which keeps $k$ small and computation cheap while performing on par with the best setting observed in the sweep.
\begin{table}[t]
\centering
\caption{Effect of the pruning parameter $\alpha$ on defender utility.}
\label{tab:alpha_sweep}
\begin{tabular}{r c}
\hline
$\alpha$ &Utility  \\
\hline
1   & $138.68 \pm 2.93$ \\
5   & $138.86 \pm 2.75$ \\
50  & $135.11 \pm 0.01$ \\
\hline
\end{tabular}
\caption{Effect of the $\alpha$ hyperparmeter on player average utility for a 1000 device network. Higher is better.}
\end{table}

\section{\texorpdfstring{$k$}{k}-hop invalidation hyperparameter ablation}

\smallskip
\noindent\textbf{Effect of the invalidation parameter $k$.}
Table~\ref{tab:k_invalidation} varies the $k$-hop invalidation radius used by the Q-value cache on payoff matrix construction. Payoff matrix construction occurs several times in the training process and dominates the majority of the computational overhead. All three settings give broadly similar player utilities, with averages in the range $127$–$139$. The most aggressive invalidation, $k = 10$, achieves the highest utility ($138.86 \pm 3.75$), while $k = 4$ underperforms and shows the largest variance. The local setting, $k = 1$, sits in between on payoff ($135.98 \pm 0.78$) but is noticeably cheaper computationally: payoff-matrix construction is about $1.28$\,ms on average, roughly a 10–15\% reduction relative to $k = 4$, and uses about 11\,MB less RAM than the larger radii. Since the utility differences are modest compared to the change in wall time and memory footprint, we adopt $k = 1$ as the default in our main experiments; it keeps cache reuse conservative, but avoids paying for wide invalidation sweeps that do not translate into clear gains in performance.

\end{document}